\documentclass[10pt,conference]{IEEEtran}
\makeatletter
 \def\blfootnote{\xdef\@thefnmark{}\@footnotetext}
 \makeatother

\usepackage{array}
\usepackage{amsthm}
\usepackage{bbm}
\usepackage[justification=centering]{caption}
\usepackage[usenames]{color}
\usepackage[noadjust]{cite}
\ifCLASSINFOpdf
   \usepackage[pdftex]{graphicx}
   \usepackage{epstopdf}
\else
  \usepackage[dvips]{graphicx}
\fi
\usepackage[cmex10]{amsmath}
\usepackage[]{amssymb}
\usepackage{algorithmic}
\usepackage{algorithm}
\usepackage{mdwmath}
\usepackage{mdwtab}
\usepackage{subfigure}
\usepackage{url}
\usepackage{color}
\usepackage[verbose]{wrapfig}
\usepackage{hyperref}

\allowdisplaybreaks[4]
\newtheorem{theorem}{Theorem}
\newtheorem{Definition}{Definition}

\newtheorem{Lemma}{Lemma}
\newtheorem{Remark}{Remark}

\newtheorem{Assumption}{Assumption}

\newtheorem*{Example*}{Example}

\begin{document}
\setlength{\abovedisplayskip}{0pt}
\setlength{\belowdisplayskip}{0pt}
\title{Decentralized Differentially Private Without-Replacement Stochastic Gradient Descent}
\newcommand*{\affaddr}[1]{#1} 
\newcommand*{\affmark}[1][*]{\textsuperscript{#1}}
\newcommand*{\email}[1]{\texttt{#1}}

\author{
    \IEEEauthorblockN{Richeng Jin, \textit{Member, IEEE}, Xiaofan He, \textit{Senior Member, IEEE},
    Huaiyu Dai, \textit{Fellow, IEEE}}
}

\maketitle
\begin{abstract}
While machine learning has achieved remarkable results in a wide variety of domains, the training of models often requires large datasets that may need to be collected from different individuals. As sensitive information may be contained in the individual’s dataset, sharing training data may lead to severe privacy concerns. Therefore, there is a compelling need to develop privacy-aware machine learning methods, for which one effective approach is to leverage the generic framework of differential privacy. Considering that stochastic gradient descent (SGD) is one of the most commonly adopted methods for large-scale machine learning problems, a decentralized differentially private SGD algorithm is proposed in this work. Particularly, we focus on SGD without replacement due to its favorable structure for practical implementation. Both privacy and convergence analysis are provided for the proposed algorithm. Finally, extensive experiments are performed to demonstrate the effectiveness of the proposed method.
\end{abstract}
{\blfootnote{R. Jin is with the Zhejiang–Singapore Innovation and AI Joint Research Lab and the College of Information Science and Electronic Engineering, Zhejiang University, Hangzhou, China, 310000 (e-mail: richengjin@zju.edu.cn).
X. He is with the Electronic Information School, Wuhan University, Wuhan, China, 430000 (e-mail: xiaofanhe@whu.edu.cn).
H. Dai is with the Department of Electrical and Computer Engineering, North Carolina State University, Raleigh, NC, USA, 27695 (e-mail: hdai@ncsu.edu).
}}

\section{Introduction}
\noindent With the rapid development of wireless sensor networks and smart devices, it is nowadays becoming easier to collaboratively collect data from multiple devices for data processing and analysis. For example, as an important emerging application, health monitoring systems have drawn a lot of attention (e.g., see \cite{esteva2019guide} and the references therein). In a health monitoring system, wearable sensors are used to collect the patients' health data, which are later utilized to develop disease prediction models through machine learning techniques. Considering the size of the systems and the sensitivity of the collected data, there is a compelling need to design efficient decentralized data processing methods. Compared to centralized data processing, the decentralized approaches mainly have two advantages. Firstly, decentralization can offer better scalability by exploiting local computational resource of the smart devices. Secondly, considering that data collected from individuals (e.g., medical and financial records) are sensitive and private, decentralized processing is able to avoid direct data sharing between individual devices and the (possibly) untrusted central node, leading to improved privacy.


Due to its simplicity and scalability, stochastic gradient descent (SGD) has been extensively studied in the literature \cite{rakhlin2012making}. SGD admits decentralized implementation by allowing the individuals to compute and share the gradients derived from their local training samples, and hence is suitable for various collaborative learning applications. However, sharing the local gradients may jeopardize the privacy of the users, since an adversary may be able to infer the private local data (e.g., the health information) from the shared gradients \cite{ren2022grnn}. With such consideration, differential privacy \cite{dwork2006calibrating} has been incorporated into SGD to guarantee a quantifiable level of privacy.

Various differentially private SGD algorithms have been proposed, among which one of the most popular approaches is adding noise to the gradients in the training process (e.g., \cite{abadi2016deep,hegedus2016distributed,gong2020privacy,kim2021federated,yu2021privacy,el2022differential} and the references therein). Most of the existing works adopt the commonly used independent and identically distributed (i.i.d.) sampling method \cite{dekel2012optimal}, in which the training examples are sampled in an i.i.d fashion during each training iteration. Nonetheless, in practical implementations of SGD algorithms, without-replacement sampling is often easier and faster, as it allows sequential data access \cite{shamir2016without}. More specifically, let $i_t$ and $[n]$ denote the index of the training sample used at time $t$ and the whole training dataset, respectively. The mathematical description for i.i.d. sampling is $P(i_t=j)=1/n, \forall j\in[n]$; for without-replacement sampling, it is $P(i_t=j)=1/(n-t+1),\forall j\in[n]/\{i_1,\cdots,i_{t-1}\}$. It has been shown that without-replacement sampling is strictly better than i.i.d. sampling after sufficiently many passes over the dataset under smoothness and strong convexity assumptions \cite{gurbuzbalaban2015random}. \cite{wu2017bolt} considers without-replacement sampling for differentially private SGD. However, it adds noise to the trained models and assumes that the data is held centrally, which cannot be generalized to the decentralized setting directly.


With such consideration, in this work, a decentralized without-replacement sampling SGD algorithm with both privacy and convergence guarantees are proposed. We consider a scenario in which multiple nodes with limited numbers of training samples aim to learn a global model over the whole dataset (i.e., all the training samples from the nodes). It is assumed that each node has two models: a local model that is only available to itself and a global model that is known to the public. At each iteration, a node decides to update either the local model or the global model. To fulfill privacy-aware decentralized SGD, each node adds noise when it updates the global model. As a result, the global model is not necessary better than the local model, especially in the high privacy requirement settings. Therefore, we leverage the deep-Q learning
framework \cite{mnih2015human} to help determine whether each node updates the global model or not during each iteration.

The remainder of this paper is organized as follows. Section~\ref{preliminaries} reviews preliminaries and notations used in this work. The problem is formulated and presented in Section~\ref{ProblemFormulation}. Section~\ref{DQCDDPSGD} presents the proposed algorithm, and its effectiveness is examined through simulations in Section~\ref{Numerical results}. Conclusions and future works are presented in Section~\ref{Conclusions and Future Works}.

\section{Preliminaries and Notations}\label{preliminaries}
\noindent In this section, we start by reviewing some important definitions and existing results.

\subsection{Machine Learning and Stochastic Gradient Descent}
\noindent Suppose that there is a training data set $\mathcal{S}=\{(x_1,y_1),\cdots,(x_n,y_n)\}$ with $n$ training instances i.i.d. sampled from a sample space $\mathcal{Z} = \mathcal{X} \times \mathcal{Y}$, where $\mathcal{X}$ is a space of feature vectors and $\mathcal{Y}$ is a label space. Let $\mathcal{W} \subseteq \mathbb{R}^{d}$ be a hypothesis space equipped with the standard inner product and 2-norm $||\cdot||$. The goal is to learn a good prediction model $h(w) \in \mathcal{F}: \mathcal{X} \rightarrow \mathcal{Y}$ which is parameterized by $w \in \mathcal{W}$. The prediction accuracy is measured by a loss function $f: \mathcal{W} \times \mathcal{Z} \rightarrow \mathbb{R}$. Given a hypothesis $w \in \mathcal{W}$ and a training sample $(x_i,y_i) \in \mathcal{S}$, we have a loss $f(w,(x_i,y_i))$. SGD \cite{rakhlin2012making} is a popular optimization algorithm, which aims to minimize the empirical risk $F(w) = \frac{1}{n}\sum_{i=1}^{n}f(w,(x_i,y_i))$ over the training dataset $\mathcal{S}$ of $n$ samples and obtain the optimal hypothesis $w^{*}=\arg\min_{w} F(w)$. For simplicity, let $f_{i}(w)=f(w,(x_i,y_i))$ for fixed $\mathcal{S}$. In each iteration, given a training sample $(x_t,y_t)$, SGD updates the hypothesis $w_t$ by:
\begin{equation}
w_{t+1} = G_{f_t,\eta_t} = w_{t} - \eta_{t}\nabla f_{t}(w_t),
\end{equation}
\noindent in which $\eta_{t}$ is the learning rate and $\nabla f_{t}(w_t) = \nabla f(w_t,(x_t,y_t))$ is the gradient. We will denote $G_{f_t,\eta_t}$ as $G_{t}$ for ease of presentation.

In order to perform the convergence analysis later, some basic properties of loss functions are defined as follows.


\begin{Definition}
Let $f:\mathcal{W} \rightarrow \mathbb{R}$ be a function: $f$ is convex if for any $u,v \in \mathcal{W}$, $f(u) \geq f(v)+\langle\nabla f(v),u-v\rangle$; $f$ is L-Lipschitz if for any $u,v \in \mathcal{W}$, $||f(u)-f(v)|| \leq L||u-v||$; $f$ is $\gamma$-strongly convex if for any $u,v \in \mathcal{W}$, $f(u)\geq f(v)+\langle\nabla f(v),u-v\rangle + \frac{\gamma}{2}||u-v||^2$; $f$ is $\mu$-smooth if for any $u,v \in \mathcal{W}$, $\nabla f(u)- \nabla f(v) \leq \mu||u-v||$.
\end{Definition}

\noindent \textbf{Example}: \textbf{Logistic Regression}. The above three parameters $(L,\gamma,\mu)$ can be derived by analyzing the specific loss function. Here, we give an example using the popular $L_2$-regularized logistic regression model with the $L_2$ regularization parameter $\lambda \leq 0$, which can also be found in \cite{wu2017bolt}. Assuming that each feature vector is normalized before processing, i.e., $||x|| \leq 1$, the loss function (for $L_2$-regularized logistic regression model) on a sample $(x,y)$ with $y \in \{+1,-1\}$ is defined as follows:
\begin{equation}\label{lossfunction}
  f(w,(x,y)) = \ln(1+\exp(-y<w,x>)) + \frac{\lambda}{2}||w||^2.
\end{equation}

If $\lambda > 0$, the loss function $f(w,(x,y))$ is strongly convex. Suppose the norm of the hypothesis is bounded by $R$, i.e., $||w|| \leq R$, then it can be proved that $L=1+\lambda R, \mu = 1+ \lambda$ and $\gamma = \lambda$. If $\lambda = 0$, the loss function is only convex, and we can deduce that $L = \mu = 1$ and $\gamma = 0$.

We now introduce some important properties of gradient descent updates that will be used in the convergence and privacy analyses of the proposed algorithm.

\begin{Definition}\label{Definition2}
Let $G:\mathcal{W} \rightarrow \mathcal{W}$ be an operator that maps a hypothesis to another hypothesis. $G$ is $\rho$-expansive if $sup_{w,w'}\frac{||G(w)-G(w')||}{||w-w'||} \leq \rho$ and $\sigma$-bounded if $sup_{w}||G(w)-w|| \leq \sigma$.
\end{Definition}


\begin{Lemma}\cite{wu2017bolt}\label{CONVEXANDSTRONGLY}
Assume that $f$ is $\mu$-smooth, if $f$ is convex, then for any $\eta \leq \frac{2}{\mu}$, $G_{f,\eta}$ is 1-expansive; if $f$ is $\gamma$-strongly convex, then for $\eta \leq \frac{1}{\mu}$, $G_{f,\eta}$ is $(1-\eta \gamma)$-expansive.
\end{Lemma}

\begin{Lemma}\label{BOUNDED}
Suppose that $f$ is L-Lipschitz, then the gradient update $G_{f,\eta}$ is ($\eta$L)-bounded.
\end{Lemma}


\begin{Lemma} (Growth Recursion \cite{hardt2015train})\label{GROWTH RECURSION}
Fix any two sequences of updates $G_1,\cdots,G_T$ and $G'_1,\cdots,G'_T$. Let $w_0=w'_0$, $w_t=G_t(w_{t-1})$ and $w'_t=G'_t(w'_{t-1})$ for $t=1,\cdots,T$. Then $||w_0-w'_0||=0$ and for $0 \le t \leq T$
\begin{equation}
||w_t-w'_t|| \leq
\begin{cases}
\rho||w_{t-1}-w'_{t-1}||, ~\text{if $G_t=G'_t$ is $\rho$-expansive.}\\
min(\rho,1)||w_{t-1}-w'_{t-1}||+2\sigma_{t},~\text{if $G_t$ is} \\
~~~~~\text{$\rho$-expansive; $G_t$ and $G'_t$ are $\sigma_{t}$-bounded.}
\end{cases}
\end{equation}

\end{Lemma}

\subsection{Differential Privacy}
\noindent In this subsection, we start by reviewing the definition of differential privacy, and then introduce the Gaussian mechanism that ensures $(\epsilon,\delta)$-differential privacy.


\begin{Definition}
A (randomized) algorithm A is said to be ($\epsilon,\delta$)-differentially private if for any neighboring datasets $S,S'$, and any event $E \subseteq Range(A)$, $Pr[A(S)\in E]\leq e^\epsilon Pr[A(S')\in E]+\delta$, in which $Range(A)$ is the codomain that consists of all the possible outputs of $A$.
\end{Definition}


\begin{theorem}\label{AddingNoise}
\cite{dwork2014algorithmic} Let $q$ be a deterministic query that maps a dataset to a vector in $\mathbb{R}^d$. For $c^2 \ge 2ln(1.25/\delta)$, adding Gaussian noise sampled according to
\begin{equation}\label{Noise}
  \mathcal{N}(0,\sigma^2); \sigma \geq \frac{c\Delta_2(q)}{\epsilon}
\end{equation}
ensures ($\epsilon,\delta$)-differential privacy for $\epsilon \in (0,1)$, in which $\Delta_2(q)=max_{S \sim S'}||q(S)-q(S')||$ is the $L_2$-sensitivity.
\end{theorem}

\section{Problem Formulation}\label{ProblemFormulation}
\noindent In this work, a network consisting of $M$ computational nodes is considered. Each node in the network has a local dataset of $\frac{n}{M}$ training samples, and the set of all the training samples from all the nodes in the network form the global training dataset. The goal of the nodes is to collaboratively learn a hypothesis $w$ that minimizes the empirical risk $F(w) = \frac{1}{n}\sum_{i=1}^{n}f(w,(x_i,y_i))$ over the whole training dataset. It is assumed that each node stores two models: a local model (i.e., a local hypothesis $w^{L}$) that is only known to itself and a global model (i.e., a global hypothesis $w^{G}$) that is shared among all the nodes in the network. \textcolor{black}{All the nodes know the index of the last node that updates the global model and are able to contact it directly. For instance, the nodes can broadcast a message to the whole network indicating the step count and their indices after they update the global model. This message contains only two integers so the communication overhead is insignificant.} At each iteration, a node randomly samples a mini-batch of training examples from its own local dataset without replacement and determines whether to use and update the global model or not. If a node decides not to update the global model, it simply updates its own local model; otherwise, it first contacts the last node that has updated the global model and fetches the latest global model. Then it updates the global model using its local model and training samples through the SGD method. Since the global model is publicly known, one can infer the training sample $(x,y)$ in (\ref{lossfunction}) given the loss function $f$, previous global model $w_t$ and the updated global model $w_{t+1}$, which leads to privacy concerns and deters the nodes from collaborating. Therefore, each node will add noise to the gradients for privacy preservation.

When the nodes update the global model, they need to contact the other nodes to obtain the latest global model, which induces communication overhead and latency during message passing. Moreover, adding noise may also induce accuracy degradation. Since each node will also learn a local model that is updated without privacy concerns, the local model may sometimes be better than the global model, especially when the privacy requirement is high (i.e., small $\epsilon$). In this sense, \textcolor{black}{each node has to learn a control policy to determine whether to update the global model or not at each iteration}.

\section{Deep-Q Learning based Collaborative Decentralized Differentially Private SGD}\label{DQCDDPSGD}
\noindent In this section, a deep-Q learning based collaborative training scheme is proposed. More specifically, the model learning process is modeled as a Markov Decision Process (MDP) \cite{puterman2014markov}, in which the collaborative nodes are the agents, the current local models and the \textcolor{black}{loss} are the states, and the action for each node is whether updating the global model or not. \textcolor{black}{Reinforcement learning (RL) based methods are commonly used to solve such MDP problems in practice due to two advantages: 1) RL methods do not require prior knowledge of the underlying system dynamics, and 2) the designer is free to choose reward metrics that best match the desired controller performance \cite{ruvolo2009optimization}. There have been some works that employ RL as the controller of optimization algorithms. For example, \cite{daniel2016learning} uses RL to determine the step size of neural network training. Inspired by the success of deep RL methods \cite{mnih2015human}, a deep-Q network is adopted to control the behavior (i.e., updating the local model or the global model) of the nodes.\footnote{\textcolor{black}{The deep-Q based method proposed in this work is only our first attempt to explore the possibility of using RL to work as a controller to guide the learning process of the collaborative nodes. The optimization of the controller remains an interesting future work.}}}

The deep-Q learning algorithm is presented in Algorithm \ref{algorithm3}, in which the nodes act as the agents, and the states of the environment are defined by the local models. There are two possible actions for each node: updating the local model or the global model. The basic idea of deep-Q learning is to approximate the action-value (Q) function in traditional Q-learning by a deep neural network. Since RL is known to be unstable when a nonlinear function approximator (i.e., neural network) is used to represent the Q-function, similar to \cite{mnih2015human}, two neural networks are created for each node. The first network $\theta_{t}$ includes all the updates in the training while the second (target) network $\theta'$ retrieves the Q values and is periodically updated to imitate the first network. Experience replay is also adopted. After each action, the experience (transition) is stored in the replay memory as a tuple of $\langle state, action, reward, next state \rangle$. During each iteration, a random mini-batch of transitions is sampled and used to update $\theta_{t}$. For each transition $(s_j,a_{j}^{m},r_{j},s_{j+1})$, the target network is used to compute the approximated target value $y_{j}=r_{j}+\gamma_{DQ}\max_{a'}\hat{Q}_{m}(s_{j+1},a',\theta')$. Based on the current network $\theta_{t}$ and the state $s_{t}$, the action $a_{t}^{m}$ is determined. 

\textcolor{black}{To this end}, a Deep-Q learning based collaborative differentially private SGD algorithm (i.e., Algorithm \ref{algorithm2}) is proposed. For node $m$, given the training samples $sample_{t}^{m}$ and the current local model $w_{t}^{l_m}$ at time $t$, it \textcolor{black}{obtains the current state $s_{t}=[w_{t}^{L_m},f(w_{t}^{L_m},sample_{t}^{m})]$} and determines to update the global model or the local model via the deep-Q network. After updating the local (or global) model, the \textcolor{black}{updated loss $f(w_{t+1}^{l_m},sample_{t}^{m})$} is used to update the deep-Q network.

\begin{algorithm}[t]
\caption{Deep-Q Learning based Collaborative Decentralized Differentially Private SGD}
\label{algorithm2}
\begin{algorithmic}
\STATE 1. Require: initial vector $w_0^{L_1},\cdots,w_0^{L_M},w_0^G$, size of local mini-batch $b$, number of nodes $M$, total number of training data samples $n$, number of iterations $T$.
\STATE 2. for $t=0,1,\cdots,T$ do
\STATE 3. ~for local nodes $m$:
\STATE 4. ~~if update, run Algorithm \ref{algorithm3} and obtain action $a^{m}_{t} \in \{Local, Global\}$
\begin{itemize}
  \item If Local, obtain the mini-batch $D_{m}(t)$, compute the gradient $\nabla f_{D_{m}(t)}(w^{L_{m}}_t)$ and update its weights $w^{L_{m}}_{t+1} = w^{L_{m}}_t - 2\eta_{t}^{L_m}\nabla f_{D_{m}(t)}(w^{L_{m}}_t)$
  \item If Global, fetch the $w^{G}_t$ from the latest global model, obtain the mini-batch $D_{m}(t)$, and compute $\nabla f_{D_{m}(t)}(w^{G}_{t})$, add noise $N_t$ \textcolor{black}{to the gradient} and then update $w^{G}_{t+1}$ and $w^{L_{m}}_{t+1}$ according to the following rule
      \begin{equation}\label{GlobalUpdate}
      \color{black}
        w^{G}_{t+1} = \frac{w^{G}_{t}+w^{L_{m}}_{t}}{2} - \eta_{t}^{L_m}(\nabla f_{D_{m}(t)}(w^{G}_{t}) + N_{t}),
      \end{equation}
      \begin{equation}
        w^{L_{m}}_{t+1} = w^{G}_{t+1}.
      \end{equation}
\end{itemize}
\STATE 5.~~end if
\STATE 6.~end for
\STATE 7.end for
\end{algorithmic}
\end{algorithm}

\begin{algorithm}
\caption{Deep-Q Learning Algorithm with input $sample_t^{m}$ and $w^{L_m}_{t}$ for node $m$}
\label{algorithm3}
\begin{algorithmic}
\STATE 1. Require: replay memory \textcolor{black}{$RM_m$}, action-value function $Q_{m}$ with weights $\theta_{t}$, target action-value function $\hat{Q}_{m}$ with weights $\theta^{'}$, the previous action of the node $a^{m}_{t-1}$, the previous loss $f_{t-1}$.
\STATE 2. Given the training sample $sample_t^{m}$, set the current state \textcolor{black}{$s_{t}=[w^{L_m}_{t},f(w_{t}^{L_m},sample_{t}^{m})]$} and the previous state \textcolor{black}{$s_{t-1}=[w^{L_m}_{t-1},f(w_{t-1}^{L_m},sample_{t-1}^{m})]$. Set the reward as $r_{t-1}=-f(w_{t}^{L_m},sample_{t-1}^{m})$}.
\STATE 3. Store transition $(s_{t-1},a^{m}_{t-1},r_{t-1},s_{t})$ in \textcolor{black}{$RM_m$}.
\STATE 4. Sample random mini-batch of transitions $(s_{j},a^{m}_{j},r_{j},s_{j+1})$ from \textcolor{black}{$RM_m$}.
\begin{equation}
Set~~y_{j} =
\begin{cases}
r_{j}, ~~if~terminates~at~step~ j+1.\\
r_{j}+\gamma_{DQ}\max_{a'}\hat{Q}_{m}(s_{j+1},a',\theta^{'}),~\text{otherwise},
\end{cases}
\end{equation}
\noindent in which $\gamma_{DQ}$ is the \textcolor{black}{discounting factor}. Perform a gradient descent step on $(y_j-Q_{m}(s_j,a_j,\theta_t))$ w.r.t the network parameter $\theta_{t}$. In addition, reset $\hat{Q}_{m}=Q_{m}$ every $C$ steps.
\STATE 5. With probability $p_{explr}$ select a random action $a^{m}_{t}$, otherwise select $a^{m}_{t}=\arg max_{a}Q(s_{t},a,\theta_{t})$.
\STATE 6. Feed $a^{m}_{t}$ to Algorithm \ref{algorithm2}.
\end{algorithmic}
\end{algorithm}

Note that in Algorithm \ref{algorithm2}, the privacy concern only exists when the nodes update the global model. In (\ref{GlobalUpdate}), there are two terms that may lead to privacy leakage: $w^{L_{m}}_{t}$ and $\nabla f_{D_{m}(t)}(w^{G}_t)$. Suppose that the latest time that node $m$ updates the global model is $t-j-1$ and therefore $w_{t-j}^{L_m}=w_{t-j}^{G}$ is publicly known, we have the following Lemma.

\begin{Lemma}\label{lemmanoiseA2}
Suppose that the loss function $f$ is L-Lipschitz, convex and $\mu$-smooth, let $D_{m}(t-j:t) \triangleq \{D_{m}(t-j),\cdots,D_{m}(t)\}, D'_{m}(t-j:t) \triangleq \{D'_{m}(t-j),\cdots,D'_{m}(t)\}$ be two neighboring datasets differing at only one sample located in the $i$-th mini-batch. For Algorithm \ref{algorithm2} with $\eta_{t}^{L_m} \leq \frac{1}{2\mu}, \forall t$, we have
\begin{equation}
  sup_{D_{m}(t-j:t) \sim D'_{m}(t-j:t)}||w_{t+1}-w'_{t+1}|| \leq \max_{k \in [t-j,t]}\frac{2\eta_{k}^{L_m}L}{b}.
\end{equation}
\end{Lemma}

\begin{proof}
See Appendix \ref{proofLemma8}.
\end{proof}

\begin{theorem}\label{TheoremA2Privacy}
Suppose that the loss function $f$ is L-Lipschitz, convex and $\mu$-smooth, if \textcolor{black}{the noise term $\eta_{t}^{L_m}N_t$} is sampled according to (\ref{Noise}), with $\Delta_2(q)=||w_{t+1}-w'_{t+1}||$ which is given by Lemma \ref{lemmanoiseA2}, then Algorithm \ref{algorithm2} is $(\epsilon,\delta)$-differentially private.
\end{theorem}
\begin{proof}
See Appendix \ref{proofTheorem5}.
\end{proof}

The following theorem shows the convergence rate of Algorithm \ref{algorithm2} for convex loss function $f$.
\begin{theorem}\label{TheoremA2Convex}
  Suppose that the hypothesis space $\mathcal{W}$ has diameter $R$, the loss function $f$ is convex and $L$-Lipschitz on $\mathcal{W}$, and $||\nabla f_{i}(w)||^2 \leq B^2, \forall w,i$. Let $p_{t,L}^{L_m}$ and $p_{t,G}^{L_m}$ denote the probabilities (given by the Deep-Q learning algorithm) that node $m$ chooses to update the local model and global model, respectively. Then for any $1\le T \leq \frac{n}{b}$, if we run Algorithm \ref{algorithm2} for $T$ iterations with step size $\eta_{t}^{L_m}=\eta$, we have

\begin{equation}
\begin{split}
  &\mathbb{E}[\frac{1}{T}\sum_{t=1}^{T}F(p_{t,L}^{L_m}w^{L_m}_{t}+p_{t,G}^{L_m}w^{G}_{t})-F(w^{*})] \\
  &\leq p_{t,L}^{L_m}\mathbb{E}[\frac{1}{T}\sum_{t=1}^{T}F(w^{L_m}_{t})]+p_{t,G}^{L_m}\mathbb{E}[\frac{1}{T}\sum_{t=1}^{T}F(w^{G}_{t})]-F(w^{*}) \\
  &\leq \frac{(M+1)R^2}{4T\eta} + \eta B^2 + \frac{4\ln(1.25/\delta)\eta L^2}{b^2\epsilon^2}\frac{\sum_{m'}|p_{t,G}^{L_{m'}}|}{T}\\
  &+\frac{2(2+12\sqrt{2})RL}{3}\bigg[\frac{\sqrt{bT}}{n}+\frac{2}{\sqrt{n}+\sqrt{n-bT}}\bigg],
\end{split}
\end{equation}
\noindent in which $F(\cdot)=\frac{1}{n}\sum_{i=1}^{n}f(\cdot)$ is the empirical risk, and $\sum_{m'}|p_{t,G}^{L_{m'}}|$ is the expected total number of time instances that the nodes update the global model.
\end{theorem}


\begin{proof}
\textcolor{black}{Please see Appendix \ref{proofTheorem6}.}
\end{proof}

\begin{Remark}
By properly selecting the step size $\eta$ (e.g., $\eta \propto \frac{1}{\sqrt{n}}$), the convergence rate is  $\mathbb{E}[\frac{1}{T}\sum_{t=1}^{T}F(p_{t,L}^{L_m}w^{L_m}_{t}+p_{t,G}^{L_m}w^{G}_{t})-F(w^{*})] \leq \mathcal{O}(\frac{1}{\sqrt{n}})$. In addition, according to the definition of $w^{*}$, $\mathbb{E}[\frac{1}{T}\sum_{t=1}^{T}F(w^{G}_{t})]-F(w^{*}) \geq 0$ and therefore $p_{t,L}^{L_m}[\mathbb{E}[\frac{1}{T}\sum_{t=1}^{T}F(w^{L_m}_{t})]-F(w^{*})] \leq \mathcal{O}(\frac{1}{\sqrt{n}})$. As a result, there exists a positive constant $p_{L}^{min} \leq p_{t,L}^{L_m}, \forall t,m$ such that $\mathbb{E}[\frac{1}{T}\sum_{t=1}^{T}F(w^{L_m}_{t})]-F(w^{*}) \leq \mathcal{O}(\frac{1}{p_{L}^{min}\sqrt{n}})$, which indicates the convergence of the local models.
\end{Remark}


For the convergence rate of Algorithm \ref{algorithm2} with $\lambda$-strongly convex loss function $f$, we add the following assumption.

\begin{Assumption}\label{A2StronglyAssumption}
At each time instance $0 \leq t \leq T$, each node updates once \textcolor{black}{(either the local model or the global model).}
\end{Assumption}

\begin{theorem}\label{TheoremA2Strongly}
Suppose that the loss function $f$ is $\gamma$-strongly convex and $L$-Lipschitz, and $||\nabla f_{i}(w)||^2 \leq B^2, \forall w,i$. For any $1\le T \leq \frac{n}{bM}$, if we run Algorithm \ref{algorithm2} for $T$ iterations with step size given by $\eta_{t}^{L_m} = \frac{1}{a\gamma t}, \forall m$, in which $a=\min\{p_{t,L}^{L_1},p_{t,G}^{L_1},\cdots,p_{t,L}^{L_M},p_{t,G}^{L_M}\}>0$, we have
\begin{equation}\label{Strongly}
\begin{split}
&\sum_{m=1}^{M}\mathbb{E}[||w^{L_m}_{t+1}-w^{*}||^2] + \mathbb{E}[||w^{G}_{t+1}-w^{*}||^2] \\
&\leq \mathcal{O}(\frac{MB^2}{a^2t}+\frac{MB^2\log t}{a^2bt}+\frac{ML^2\ln(\frac{1.25}{\delta})}{a^2b^2\epsilon^2t}).
\end{split}
\end{equation}
\end{theorem}

\begin{proof}
\textcolor{black}{Please see Appendix \ref{proofTheorem7}.}
\end{proof}

\begin{Remark}
Note that the parameter $a$ may depend on the exploration rate $p_{explr}$ in the deep-Q learning algorithm, which is initialized to be large and then annealed down to a small constant (e.g., 0.1). \textcolor{black}{In particular, since there is a probability of $p_{explr}$ with which a node will randomly select an action, we have $\frac{p_{explr}}{2} \leq a \leq 1-\frac{p_{explr}}{2}$.}
\end{Remark}

\begin{figure*}[h]
\begin{minipage}[t]{0.24\linewidth}
\centering
\includegraphics[width=1\textwidth]{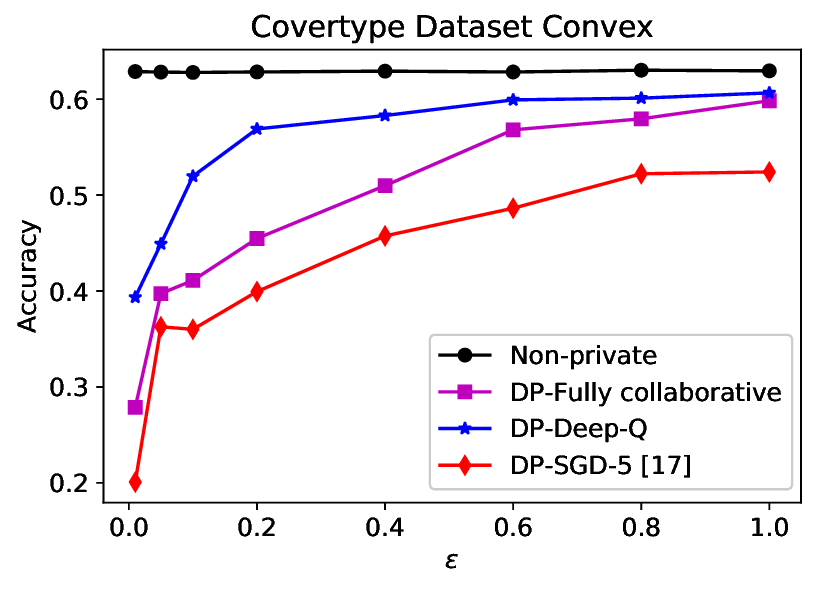}\vspace{-0.1in}
\caption{\footnotesize{Covertype Dataset Convex}}\vspace{-0.2in}
\label{Artificial Dataset Convex}
\end{minipage}
\begin{minipage}[t]{0.24\linewidth}
\centering
\includegraphics[width=1\textwidth]{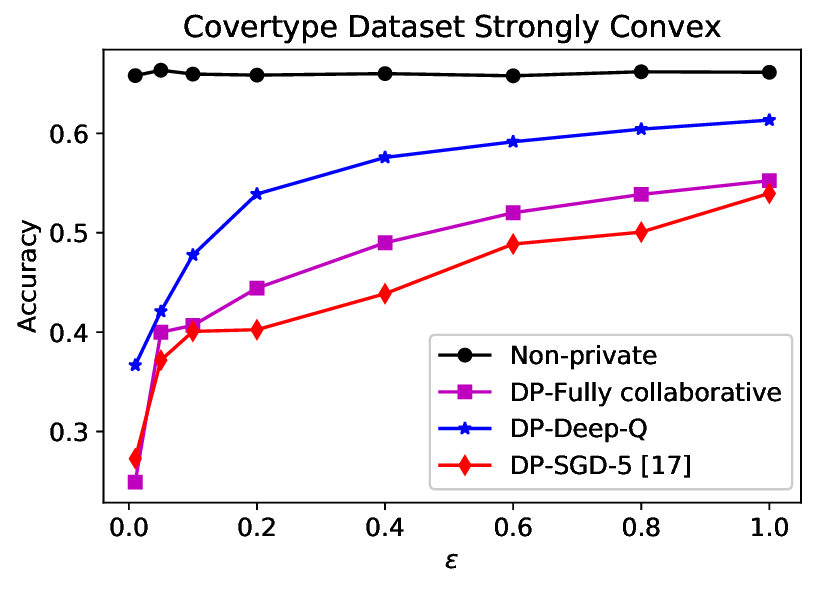}\vspace{-0.1in}
\caption{\footnotesize{Covertype Dataset Strongly Convex}}\vspace{-0.2in}
\label{Artificial Dataset Strongly Convex}
\end{minipage}
\begin{minipage}[t]{0.24\linewidth}
\centering
\includegraphics[width=1\textwidth]{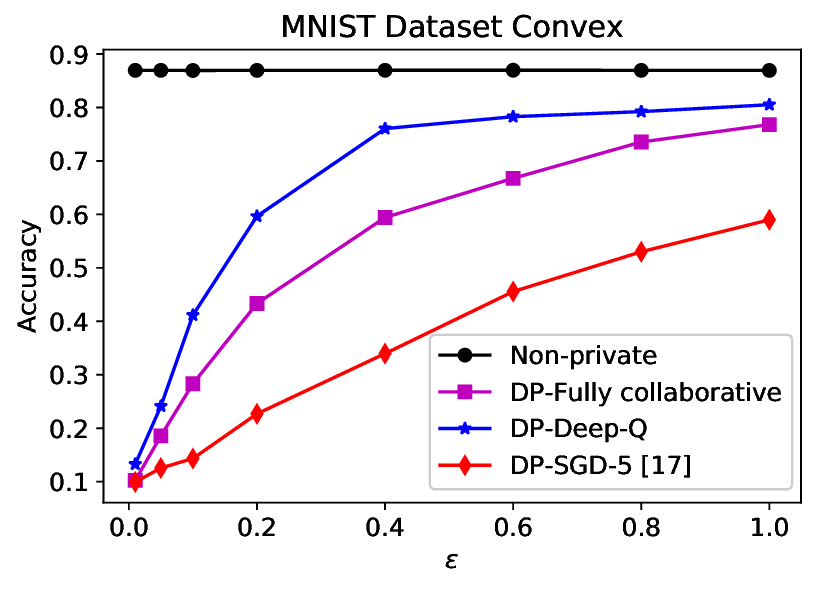}\vspace{-0.1in}
\caption{\footnotesize{MNIST Dataset Convex}}\vspace{-0.2in}
\label{MNISTConvex Dataset Convex}
\end{minipage}
\begin{minipage}[t]{0.24\linewidth}
\centering
\includegraphics[width=1\textwidth]{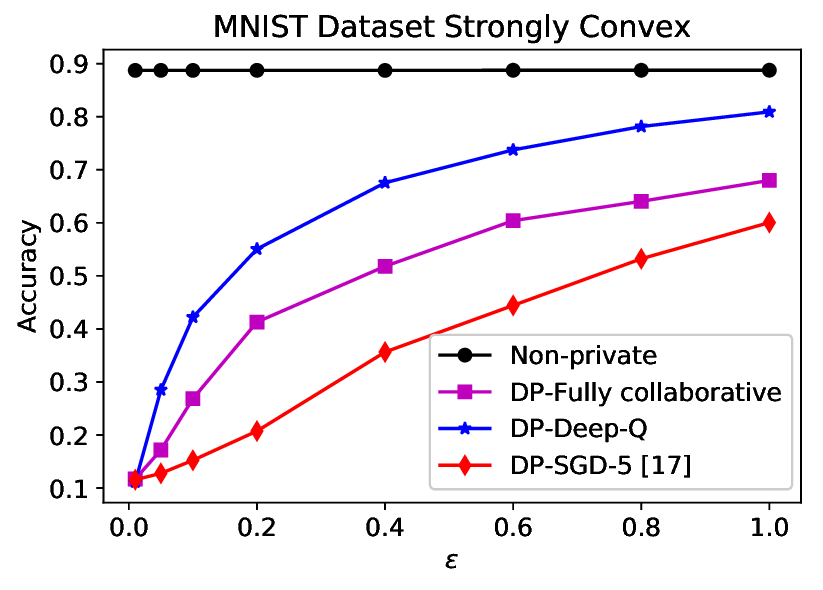}\vspace{-0.1in}
\caption{\footnotesize{MNIST Dataset Strongly Convex}}\vspace{-0.2in}
\label{MNISTConvex Dataset Strongly Convex}
\end{minipage}
\end{figure*}
\section{Simulation Results}\label{Numerical results}
\noindent This section presents simulation results to evaluate the effectiveness of the proposed algorithms. In particular, two \textcolor{black}{ widely used public datasets are considered: Covertype and MNIST.} MNIST is a computer vision dataset which consists of 70,000 $28\times28$ pixel images of handwritten digits from 0 to 9 \textcolor{black}{while Covertype is a larger dataset with 581,012 data points and a dimension of 54.} Without loss of generality, we reduce the data samples in MNIST to 50 dimensions with principal component analysis (PCA) \cite{wold1987principal} in our simulation. \textcolor{black}{In addition, the data of both datasets are normalized and projected on the surface of the unit ball before
training.} For the Deep-Q network, we build a 3-layer fully connected deep neural network for each node and choose the parameters according to \cite{mnih2015human}. The input layer consists of \textcolor{black}{$d+2$} neurons, where $d$ is the dimension of the training samples; the hidden layer consists of 128 neurons and the output layer consists of 2 neurons. The activation functions of all the three layers are linear and the weights are initialized by performing Xavier initialization in Tensorflow. The exploration rate $p_{explr}$ is set to 1 in the beginning and then annealed down to 0.1 within $\frac{n}{2Mb}$ steps. The Adam optimizer is used to train the Deep-Q neural network with a learning rate of $\gamma_{DQ}=0.01$; the mini-batch size and the size of the replay memory $D_m$ are set to 10 and 20, respectively. In the simulation, the MNIST dataset is divided into a training subset of 60,000 samples and a testing subset of 10,000 samples \textcolor{black}{while the Covertype dataset is divided into a training subset of 464,809 samples and a testing subset of 116,202 samples.} Each node randomly draws $\frac{n}{M}$ samples from the training subset as its local training dataset. We build ``one vs. all" multi-class logistic regression models for \textcolor{black}{both datasets}.\footnote{This means that 10 \textcolor{black}{(7)} binary models (one for each digit) are constructed and the output with the highest confidence is chosen as the prediction for the MNIST \textcolor{black}{(Covertype)} dataset.} Then the nodes run the proposed algorithms (one pass over their local training dataset) to train the models, followed by the testing.

\subsection{The Impact of Privacy Requirement}
\noindent In this subsection, we investigate the impact of privacy requirement on the accuracy of the proposed algorithm. It is assumed that there are 10 collaborative nodes with \textcolor{black}{60,000} training samples \textcolor{black}{for both datasets}. The privacy parameter $\delta$ is set to \textcolor{black}{$\frac{1}{n^2}$}. We set $\eta_t = 0.1$ for the convex case. For the strongly convex case, the regularization parameter and the diameter of weights $w$ is set to $\lambda=0.0001$ and $R=1/\lambda$, respectively. The mini-batch size is set to $b=50$.


Figure \ref{Artificial Dataset Convex} and Figure \ref{Artificial Dataset Strongly Convex} show the classification accuracy of the proposed algorithm for the \textcolor{black}{Covertype} dataset in the convex and strongly convex scenarios, respectively. More specifically, the simulation results of four scenarios are presented: the fully collaborative and noiseless case (denoted as ``Noiseless"); the differentially private and fully collaborative case (i.e., the nodes update the global model with probability 1, denoted as ``DP-Fully collaborative"); the differentially private and Deep-Q learning based algorithm (i.e., Algorithm \ref{algorithm2}, denoted as ``DP-Deep-Q"); \textcolor{black}{the baseline DP-SGD algorithm that adopts the i.i.d sampling strategy (denoted as ``DP-SGD-With Replacement"). For ``DP-SGD-With Replacement", we select a node to update a global model uniformly at random during each iteration. Each iteration is ensured to be $(\frac{\epsilon}{5},\frac{\delta}{5})$ differentially private and the nodes stop updating the global model once their privacy budgets are depleted (i.e., the training samples have been visited 5 times).  In addition, we use the same learning rate as that in \cite{hegedus2016distributed} and set $\eta_t=\frac{1}{\sqrt{t}}$}. \textcolor{black}{It can be observed that Algorithm \ref{algorithm2} outperforms both ``DP-Fully collaborative" and ``DP-SGD-With Replacement". While ``DP-Fully collaborative" gives higher accuracy than ``DP-SGD-With Replacement", another improvement of up to 10\% in accuracy can be achieved by using the Deep-Q learning based algorithm in both convex and strongly convex scenarios.} Similar results are observed on the MNIST dataset in Figure \ref{MNISTConvex Dataset Convex} and Figure \ref{MNISTConvex Dataset Strongly Convex}.

\begin{table}[!t]
\centering
\caption{\footnotesize{The Accuracy of Proposed Algorithms for Covertype Dataset}}
\label{tableArtificial}
\renewcommand{\arraystretch}{0.2}
\begin{tabular}{ | m{8em} | m{0.7cm}| m{0.7cm} | m{0.7cm} | m{0.7cm} |m{0.7cm}|}
\hline
Number of nodes & 1 & 3 & 5 & 10 & 20\\
\hline
Noiseless (convex)    & 56.24\% & 61.64\% & 61.90\% & 62.83\% &64.10\%\\
\hline
Fully Collaborative (convex)    & 54.21\% & 58.22\% & 58.63\% & 59.82\% &60.22\%\\
\hline
Algorithm \ref{algorithm2} (convex)    & - & 59.05\% & 59.22\% & 60.65\% &61.23\%\\
\hline
Noiseless (strongly convex)  & 62.79\% & 64.44\% & 65.31\% & 65.96\% &66.04\%\\
\hline
Fully Collaborative (strongly convex) & 50.60\% & 51.68\% & 52.60\% & 55.24\% &57.73\%\\
\hline
Algorithm \ref{algorithm2} (strongly convex) & - & 55.25\% & 59.27\% & 61.34\% &62.37\%\\
\hline
\end{tabular}
\end{table}
\vspace{-3mm}
\begin{table}[!t]
\centering
\caption{\footnotesize{The Accuracy of Proposed Algorithms for MNIST Dataset }}
\label{tableMnist}
\renewcommand{\arraystretch}{0.2}
\begin{tabular}{ | m{8em} | m{0.7cm}| m{0.7cm} | m{0.7cm} | m{0.7cm} |m{0.7cm}|}
\hline
Number of nodes & 1 & 3 & 5 & 10 &20\\
\hline
Noiseless (convex)    & 77.74\% & 84.07\% & 85.49\% & 86.83\% &87.69\%\\
\hline
Fully Collaborative (convex)    & 63.86\% & 71.48\% & 74.04\% & 76.80\% &78.17\%\\
\hline
Algorithm \ref{algorithm2} (convex)    & - & 77.17\% & 79.87\% & 80.52\% &81.93\%\\
\hline
Noiseless (strongly convex)  & 84.80\% & 88.13\% & 88.51\% & 88.76\% &88.96\%\\
\hline
Fully Collaborative (strongly convex) & 55.04\% & 63.46\% & 65.63\% & 68.00\% &73.9\%\\
\hline
Algorithm \ref{algorithm2} (strongly convex) & - & 71.03\% & 75.39\% & 80.93\% &82.97\%\\
\hline
\end{tabular}
\end{table}

\subsection{The Impact of the Number of Participating Nodes}

\noindent In this subsection, we investigate the impact of the number of participating nodes. In particular, it is assumed that each node has \textcolor{black}{60,000 training samples for both datasets.} Table \ref{tableArtificial} and Table \ref{tableMnist} show the accuracy of the proposed algorithms in different scenarios for the \textcolor{black}{Covertype} dataset and the MNIST dataset with $\epsilon=1$, respectively. It can be observed that as the number of participating nodes grows, the accuracy for both ``DP-Fully collaborative" and Algorithm \ref{algorithm2} increases since there are more training samples in total which can reduce the impact of the noise added at each iteration. In addition, Algorithm \ref{algorithm2} is always better than ``DP-Fully collaborative" and as the number of collaborative nodes grows, they are expected to approach the performance of the noiseless case. In the simulated scenarios, the accuracy degradation induced by privacy is within \textcolor{black}{6\% and 4\% for Algorithm \ref{algorithm2} when there are 20 collaborative nodes for the MNIST dataset and the Covertype dataset, respectively.} 

\section{Conclusions and Future Works}\label{Conclusions and Future Works}
\noindent In this work, the scenario in which multiple nodes (with limited training samples) collaboratively learn a global model is studied. A decentralized differentially private without-replacement SGD algorithm is proposed, and both privacy and convergence analysis are provided. Extensive simulations are conducted to demonstrate the effectiveness of the proposed algorithm. Since we only consider the cases in which the objective functions are convex, differentially private non-convex optimization problems remain our future work.

\appendices
\setlength{\abovedisplayskip}{1pt}
\setlength{\belowdisplayskip}{1pt}

\section{Proof of Lemma \ref{lemmanoiseA2}}\label{proofLemma8}
\begin{proof}
Since node $m$ updates the global model at time $t-j-1$, (\ref{GlobalUpdate}) can be written as follows:
\begin{equation}
\begin{split}
   w^{G}_{t+1} = &\frac{w^{G}_{t}+w^{L_{m}}_{t-j}-\sum_{k=1}^{j}2\eta_{t-k}^{L_m}\nabla f_{D_{m}(t-k)}(w^{L_{m}}_{t-k})}{2} \\
   &-\eta_{t}^{L_m}\nabla f_{D_{m}(t)}(w^{G}_{t}) + N_{t},
\end{split}
\end{equation}
\noindent in which $D_{m}(t-k)$ is empty, and therefore, $\nabla f_{D_{m}(t-k)}(w^{L_{m}}_{t-k})=0$ if node $m$ does not update its local model at time $t-k$ either. Since $D_{m}(t-j:t)$ and $D'_{m}(t-j:t)$ differs at only the $i-th$ mini-batch, there are two possible cases.

\textbf{case 1: ($i=t$)} In this case, we have
\begin{equation}\label{case1equation1}
\begin{split}
||w^{G}_{t+1}-w_{t+1}^{G'}||&=\eta_{t}^{L_m}||\nabla f_{D_{m}(t)}(w^{G}_{t}) - \nabla f_{D_{m}^{‘}(t)}(w^{G}_{t})|| \\
&\leq \frac{2\eta_{t}^{L_m}L}{b},
\end{split}
\end{equation}

\textbf{case 2: ($i \in [t-j,t)$)} In this case,
\begin{equation}
\begin{split}
||w^{G}_{t+1}-w_{t+1}^{G'}|| = \frac{1}{2}||w_{t}^{L_{m}}-w_{t}^{L'_{m}}||,
\end{split}
\end{equation}

\noindent in which $w_{t}^{L_{m}}$ and $w_{t}^{L'_{m}}$ are the local models of node $m$ after $j$ updates using the local mini-batches $D_{m}(t-j:t-1)$ and $D'_{m}(t-j:t-1)$, respectively. According to Lemma \ref{CONVEXANDSTRONGLY}-\ref{GROWTH RECURSION}, when $f_{i}$'s are convex, we have
\begin{equation}
\begin{aligned}
&||w_{k}^{L_{m}}-w_{k}^{L'_{m}}|| \leq \\ \nonumber
&\begin{cases}
||w_{k-1}^{L_{m}}-w_{k-1}^{L'_{m}}||, &\text{if} ~D_{m}(k-1)=D_{m}^{'}(k-1).\\
||w_{k-1}^{L_{m}}-w_{k-1}^{L'_{m}}|| + \frac{4\eta_{k-1}^{L_m}L}{b}, &\text{if} ~D_{m}(k-1)\neq D'_{m}(k-1).
\end{cases}
\end{aligned}
\end{equation}
As a result,
\begin{equation}\label{case2equation2}
\begin{split}
\frac{1}{2}||w_{t}^{L_{m}}-w_{t}^{L'_{m}}|| \leq \max_{k \in [t-j,t)}\frac{2\eta_{k}^{L_m}L}{b}.
\end{split}
\end{equation}
Combining (\ref{case1equation1}) and (\ref{case2equation2}), we have
\begin{equation}\label{noiseA2}
||w^{G}_{t+1}-w_{t+1}^{G'}|| \leq \max_{k \in [t-j,t]}\frac{2\eta_{k}^{L_m}L}{b}.
\end{equation}
\end{proof}

\section{Proof of Theorem \ref{TheoremA2Privacy}}\label{proofTheorem5}
\begin{proof}
Combing Lemma \ref{lemmanoiseA2} and Theorem \ref{AddingNoise}, it follows that each update step in Algorithm \ref{algorithm2} is $(\epsilon,\delta)$-differentially private. Since each mini-batch is only visited once, Algorithm \ref{algorithm2} is also $(\epsilon,\delta)$-differentially private over the whole dataset.
\end{proof}

\section{Proof of Theorem \ref{TheoremA2Convex}}\label{proofTheorem6}
\begin{proof}
Suppose that node $m$ obtains a mini-batch of training sample and decides to update either its local model or the global model at time $t$. Let $\hat{w}_{t}=p_{t,L}^{L_m}w^{L_m}_{t}+p_{t,G}^{L_m}w^{G}_{t}$, $f_{D_{m}(t)}(\hat{w}_{t})=\frac{1}{b}\sum_{i\in D_{m}(t)}f_i(\hat{w}_{t})$ we have
\begin{equation}\label{A2ConvexE6}
 \begin{split}
    &\mathbb{E}[\frac{1}{T}\sum_{t=1}^{T}F(\hat{w}_{t})-F(w^{*})] \\
      &=\mathbb{E}[\frac{1}{T}\sum_{t=1}^{T}(F(\hat{w}_{t})-f_{D_{m}(t)}(\hat{w}_{t}))]\\
      &+\mathbb{E}[\frac{1}{T}\sum_{t=1}^{T}f_{D_{m}(t)}(\hat{w}_{t})-F(w^{*})] \\
      &=\mathbb{E}[\frac{1}{T}\sum_{t=1}^{T}(F(\hat{w}_{t})-f_{D_{m}(t)}(\hat{w}_{t}))]\\
      &+\mathbb{E}\big[\frac{1}{T}\sum_{t=1}^{T}[f_{D_{m}(t)}(\hat{w}_{t})-f_{D_{m}(t)}(w^{*})]\big].
 \end{split}
\end{equation}
We bound the second term first. According to the update rule, we have
\begin{equation}\label{A2ConvexE5}
\begin{split}
&\mathbb{E}[||w^{G}_{t+1}-w^{*}||^2]+\mathbb{E}[||w^{L_m}_{t+1}-w^{*}||^2]\\
&=p_{t,L}^{L_m}\bigg[\mathbb{E}[||w^{L_m}_{t}-2\eta_{t}^{L_m}\nabla f_{D_{m}(t)}(w_{t}^{L_m})-w^{*}||^2] \\
&+\mathbb{E}[||w^{G}_{t}-w^{*}||^2]\bigg] \\
&+2p_{t,G}^{L_m}\mathbb{E}[||\frac{w^{G}_{t}+w^{L_{m}}_{t}}{2} - \eta_{t}^{L_m}\nabla f_{D_{m}(t)}(w^{G}_{t})-w^{*} + N_{t}||^2]
\end{split}
\end{equation}

In particular, the first term of (\ref{A2ConvexE5}) admits
\begin{equation}\label{A2ConvexE1}
\begin{split}
&\mathbb{E}[||w^{L_m}_{t}-2\eta_{t}^{L_m}\nabla f_{D_{m}(t)}(w_{t}^{L_m})-w^{*}||^2 \\
&=\mathbb{E}[||w^{L_m}_{t} - w^{*}||^2] + 4(\eta_{t}^{L_m})^{2}\mathbb{E}[||\nabla f_{D_{m}(t)}(w_{t}^{L_m})||^2] \\
&-4\eta_{t}^{L_m}\mathbb{E}[<w^{L_m}_{t} - w^{*},\nabla f_{D_{m}(t)}(w_{t}^{L_m})>].
\end{split}
\end{equation}

\noindent and the second term of (\ref{A2ConvexE5}) admits

\begin{equation}\label{A2ConvexE2}
\begin{split}
&\mathbb{E}[||\frac{w^{G}_{t}+w^{L_{m}}_{t}}{2} - \eta_{t}^{L_m}\nabla f_{D_{m}(t)}(w^{G}_{t})-w^{*} + N_{t}||^2] \\
&=\mathbb{E}[||\frac{w^{G}_{t}+w^{L_{m}}_{t}}{2} - \eta_{t}^{L_m}\nabla f_{D_{m}(t)}(w^{G}_{t})-w^{*}||^2] + \mathbb{E}[||N_{t}||^2] \\
&\leq 2\mathbb{E}[||\frac{w^{G}_{t}-w^{*}}{2}- \eta_{t}^{L_m}\nabla f_{D_{m}(t)}(w^{G}_{t})||^2] \\
&+ 2\mathbb{E}[||\frac{w^{L_m}_{t}-w^{*}}{2}||^2] + \mathbb{E}[||N_{t}||^2] \\
&\leq 2\mathbb{E}[||\frac{w^{G}_{t}-w^{*}}{2}||^2]+2(\eta_{t}^{L_m})^{2}\mathbb{E}[||\nabla f_{D_{m}(t)}(w^{G}_{t})||^2] \\
&+2\mathbb{E}[||\frac{w^{L_m}_{t}-w^{*}}{2}||^2] + \mathbb{E}[||N_{t}||^2]\\
&-4\eta_{t}^{L_m}\mathbb{E}[<\frac{w^{G}_{t}-w^{*}}{2},\nabla f_{D_{m}(t)}(w^{G}_{t})>].
\end{split}
\end{equation}
\noindent in which the first equality is due to the fact that $N_t$ is zero-mean Gaussian noise.

Due to convexity, we have
\begin{equation}\label{A2ConvexE3}
\begin{split}
&<w^{L_m}_{t}-w^{*},\nabla f_{D_{m}(t)}(w^{L_{m}}_{t})> \geq f_{D_{m}(t)}(w^{L_m}_{t}) - f_{D_{m}(t)}(w^{*})
\end{split}
\end{equation}
\begin{equation}\label{A2ConvexE4}
\begin{split}
&<w^{G}_{t}-w^{*},\nabla f_{D_{m}(t)}(w^{G}_{t})> \geq f_{D_{m}(t)}(w^{G}_{t}) - f_{D_{m}(t)}(w^{*})
\end{split}
\end{equation}
Plugging (\ref{A2ConvexE1}), (\ref{A2ConvexE2}), (\ref{A2ConvexE3}) and (\ref{A2ConvexE4}) into (\ref{A2ConvexE5}) yields

\begin{equation}
\begin{split}
&\mathbb{E}[f_{D_{m}(t)}(p_{t,L}^{L_m}w^{L_m}_{t}+p_{t,G}^{L_m}w^{G}_{t})] - f_{D_{m}(t)}(^{*})\\
&\leq p_{t,L}^{L_m}\mathbb{E}[f_{D_{m}(t)}(w^{L_m}_{t})- f_{D_{m}(t)}(w^{*})] \\
&+p_{t,G}^{L_m}\mathbb{E}[f_{D_{m}(t)}(w^{G}_{t})- f_{D_{m}(t)}(w^{*})] \\
&\leq \frac{1}{4\eta_{t}^{L_m}}\bigg[\mathbb{E}[||w^{L_m}_{t}-w^{*}||^2]+\mathbb{E}[||w^{G}_{t}-w^{*}||^2] \\
&- \mathbb{E}[||w^{L_m}_{t+1}-w^{*}||^2] - \mathbb{E}[||w^{G}_{t+1}-w^{*}||^2] \bigg]\\
&+\eta_{t}^{L_m}B^2+\frac{1}{2\eta_{t}^{L_m}}p_{t,G}^{L_m}\mathbb{E}[||N_{t}||^2].
\end{split}
\end{equation}

Let $\eta_1^{L_m}=\eta_2^{L_m}=\cdots=\eta_{T}^{L_m}=\eta$. We have
\begin{equation}
\mathbb{E}[||N_{t}||^2] = \frac{8\ln(1.25/\delta)\eta^2L^2}{b^2\epsilon^2}
\end{equation}

Averaging both sides over $t=1,\cdots,T$, we have
\begin{equation}
\begin{split}
&\mathbb{E}[\frac{1}{T}\sum_{t=1}^{T}[f_{D_{m}(t)}(\hat{w}_{t})-f_{D_{m}(t)}(w^{*})]]\\
&\leq \frac{(M+1)R^2}{4T\eta} + \eta B^2 + \frac{4\ln(1.25/\delta)\eta L^2}{b^2\epsilon^2}\frac{\sum_{m'}|p_{t,G}^{L_{m'}}|}{T},
\end{split}
\end{equation}
\noindent in which $\sum_{m'}|p_{t,G}^{L_{m'}}|$ is the expected total number of time instances that the nodes update the global model.

Then we try to bound the first term, since $f_{i}$ is $L$-Lipschitz, we have $sup_{\boldsymbol{w} \in \mathcal{W}}|| f_i(\boldsymbol{w})||\leq LR$. According to Lemma \ref{lemma1}, Lemma \ref{lemmameng1} and Lemma \ref{lemmameng2} (see Appendix \ref{TRC})
\begin{equation}\label{bound}
\begin{split}
&\mathbb{E}[\frac{1}{T}\sum_{t=1}^{T}(F(\hat{w}_{t})-f_{D_{m}(t)}(\hat{w}_{t}))] \\ &=\frac{1}{T}\sum_{t=1}^{T}\frac{(t-1)b}{n}\mathbb{E}[F_{1:(t-1)b}-F_{(t-1)b+1:n}] \\
&\leq \frac{(2+12\sqrt{2})LR}{T}\sum_{t=2}^{T}\frac{b(t-1)}{n}(\frac{1}{\sqrt{(t-1)b}}+\frac{1}{\sqrt{(n-(t-1)b)}})\\
&\leq \frac{(2+12\sqrt{2})bLR}{Tn}\int_{t=0}^{T}(\sqrt{\frac{t}{b}}+\frac{t}{\sqrt{n-tb}})\\
&=\frac{(2+12\sqrt{2})bLR}{Tn}\times \\
&\bigg[\sqrt{\frac{1}{b}}\frac{2}{3}T^{\frac{3}{2}}+\frac{2}{3b^2}[2n\sqrt{n}-\sqrt{n-bT}(2n+bT)]\bigg]\\
&=\frac{2(2+12\sqrt{2})LR}{3}\bigg[\frac{\sqrt{bT}}{n}+\frac{2}{T}(\frac{\sqrt{n}}{b}-\sqrt{n-bT}(\frac{1}{b}+\frac{T}{2n}))\bigg]\\
&\leq \frac{2(2+12\sqrt{2})LR}{3}\bigg[\frac{\sqrt{bT}}{n}+\frac{2}{T}(\frac{\sqrt{n}}{b}-\sqrt{n-bT}\frac{1}{b})\bigg]\\
&= \frac{2(2+12\sqrt{2})LR}{3}\bigg[\frac{\sqrt{bT}}{n}+\frac{2}{bT}(\frac{bT}{\sqrt{n}+\sqrt{n-bT}})\bigg]\\
&= \frac{2(2+12\sqrt{2})LR}{3}\bigg[\frac{\sqrt{bT}}{n}+\frac{2}{\sqrt{n}+\sqrt{n-bT}}\bigg]\\
\end{split}
\end{equation}

As a result, we have
\begin{equation}
\begin{split}
 &\mathbb{E}[\frac{1}{T}\sum_{t=1}^{T}F(p_{t,L}^{L_m}w^{L_m}_{t}+p_{t,G}^{L_m}w^{G}_{t})-F(w^{*})] \\
 &\leq p_{t,L}^{L_m}\mathbb{E}[\frac{1}{T}\sum_{t=1}^{T}F(w^{L_m}_{t})]+p_{t,G}^{L_m}\mathbb{E}[\frac{1}{T}\sum_{t=1}^{T}F(w^{G}_{t})]-F(w^{*})\\
 &\leq \frac{M+1}{4T\eta} + \eta B^2 + \frac{4\ln(1.25/\delta)\eta L^2}{b^2\epsilon^2}\frac{\sum_{m}|p_{t,G}^{L_m}|}{T}\\
 &+\frac{2(2+12\sqrt{2})RL}{3}\bigg[\frac{\sqrt{bT}}{n}+\frac{2}{\sqrt{n}+\sqrt{n-bT}}\bigg].
\end{split}
\end{equation}
\end{proof}

\section{Proof of Theorem \ref{TheoremA2Strongly}}\label{proofTheorem7}

\begin{proof}
Note that in this case, the global model that node $m$ uses to update may not be $w^G_{t}$ since it may already be updated by the other nodes. Therefore, let $w^G_{t+\frac{1}{2}}$ denote the global model which is utilized by node $m$ at time $t$, we have
\begin{equation}\label{4444}
\begin{split}
&\mathbb{E}[||w^{L_m}_{t+1}-w^{*}||^2] + \mathbb{E}[||w^{G}_{t+1}-w^{*}||^2] \\
&=2p_{t,G}^{L_m}\mathbb{E}\bigg[||\frac{w^{G}_{t+\frac{1}{2}}+w^{L_m}_{t}}{2}-\eta_{t}^{L_m}\nabla f_{D_{m}(t)}(w^{G}_{t+\frac{1}{2}})-w^{*} \\
&+N_{t}^{m}||^2\bigg]+p_{t,L}^{L_m}\mathbb{E}[||w^{G}_{t+\frac{1}{2}}-w^{*}||^2]\\
&+p_{t,L}^{L_m}\mathbb{E}[||w^{L_m}_{t}-2\eta_{t}^{L_m}\nabla f_{D_{m}(t)}(w^{L_m}_{t})-w^{*}||^2].
\end{split}
\end{equation}

In particular, we have
\begin{equation}\label{1111}
\begin{split}
&\mathbb{E}[||w^{L_m}_{t}-2\eta_{t}^{L_m}\nabla f_{D_{m}(t)}(w^{L_m}_{t})-w^{*}||^2]\\
&\leq \mathbb{E}[||w^{L_m}_{t}-w^{*}||^2] + (2\eta_{t}^{L_m})^2\mathbb{E}[||\nabla f_{D_{m}(t)}(w^{L_m}_{t})||^2]\\
&-4\eta_{t}^{L_m}\mathbb{E}[<w^{L_m}_{t}-w^{*},\nabla f_{D_{m}(t)}(w^{L_m}_{t})>]\\
&\leq \mathbb{E}[||w^{L_m}_{t}-w^{*}||^2] + 8(\eta_{t}^{L_m})^2\mathbb{E}[||F(w^{L_m}_{t})||^2]\\
&+8(\eta_{t}^{L_m})^2\mathbb{E}[||F(w^{L_m}_{t})-\nabla f_{D_{m}(t)}(w^{L_m}_{t})||^2]\\
&-4\eta_{t}^{L_m}\mathbb{E}[<w^{L_m}_{t}-w^{*},\nabla f_{D_{m}(t)}(w^{L_m}_{t})>],
\end{split}
\end{equation}

\begin{equation}\label{2222}
\begin{split}
&\mathbb{E}\bigg[||\frac{w^{G}_{t+\frac{1}{2}}+w^{L_m}_{t}}{2}-\eta_{t}^{L_m}\nabla f_{D_{m}(t)}(w^{G}_{t+\frac{1}{2}})-w^{*}-N_{t}^{m}||^2\bigg]\\
&\leq 2\mathbb{E}[||\frac{w^{G}_{t+\frac{1}{2}}-w^{*}}{2}||^2] + 2(\eta_{t}^{L_m})^2\mathbb{E}[||\nabla f_{D_{m}(t)}(w^{G}_{t+\frac{1}{2}})||^2]\\
&-4\eta_{t}^{L_m}\mathbb{E}[<\frac{w^{G}_{t+\frac{1}{2}}-w^{*}}{2},\nabla f_{D_{m}(t)}(w^{G}_{t+\frac{1}{2}})>] \\
&+ 2\mathbb{E}[||\frac{w^{L_m}_{t}-w^{*}}{2}||^2] + \mathbb{E}[||N_{t}^{m}||^2] \\
&\leq 2\mathbb{E}[||\frac{w^{G}_{t+\frac{1}{2}}-w^{*}}{2}||^2] + 4(\eta_{t}^{L_m})^2\mathbb{E}[||F(w^{G}_{t+\frac{1}{2}})||^2]\\
&+4(\eta_{t}^{L_m})^2\mathbb{E}[||F(w^{G}_{t+\frac{1}{2}})-\nabla f_{D_{m}(t)}(w^{G}_{t+\frac{1}{2}})||^2]\\
&-4\eta_{t}^{L_m}\mathbb{E}[<\frac{w^{G}_{t+\frac{1}{2}}-w^{*}}{2},\nabla f_{D_{m}(t)}(w^{G}_{t+\frac{1}{2}})>] \\
&+ 2\mathbb{E}[||\frac{w^{L_m}_{t}-w^{*}}{2}||^2] + \mathbb{E}[||N_{t}^{m}||^2].
\end{split}
\end{equation}

In addition, according to the strong convexity, we have
\begin{equation}\label{6666}
\begin{split}
&-\mathbb{E}[<w^{L_m}_{t}-w^{*},\nabla f_{D_{m}(t)}(w^{L_m}_{t})>] \\
&=-\mathbb{E}[<w^{L_m}_{t}-w^{*},\nabla F(w^{L_m}_{t})>] \\
&+\mathbb{E}[<w^{L_m}_{t}-w^{*},\nabla F(w^{L_m}_{t}) - \nabla f_{D_{m}(t)}(w^{L_m}_{t})>] \\
&\leq -\gamma\mathbb{E}[||w^{L_m}_{t}-w^{*}||^2]+\frac{\gamma}{2}\mathbb{E}[||w^{L_m}_{t}-w^{*}||^2] \\
&+\frac{1}{2\gamma}\mathbb{E}[||\nabla F(w^{L_m}_{t}) - \nabla f_{D_{m}(t)}(w^{L_m}_{t})||^2] \\
&=\frac{1}{2\gamma}\mathbb{E}[||\nabla F(w^{L_m}_{t}) - \nabla f_{D_{m}(t)}(w^{L_m}_{t})||^2] \\
&-\frac{\gamma}{2}\mathbb{E}[||w^{L_m}_{t}-w^{*}||^2],
\end{split}
\end{equation}

and similarly,

\begin{equation}\label{7777}
\begin{split}
&-\mathbb{E}[<w^{G}_{t+\frac{1}{2}}-w^{*},\nabla f_{D_{m}(t)}(w^{G}_{t+\frac{1}{2}})>] \\
&\leq \frac{1}{2\gamma}\mathbb{E}[||\nabla F(w^{G}_{t+\frac{1}{2}})-\nabla f_{D_{m}(t)}(w^{G}_{t+\frac{1}{2}})||^2]\\
&-\frac{\gamma}{2}\mathbb{E}[||w^{G}_{t+\frac{1}{2}}-w^{*}||^2].
\end{split}
\end{equation}

Plugging (\ref{1111}),(\ref{2222}),(\ref{6666}) and (\ref{7777}) into (\ref{4444}) gives

\begin{equation}\label{Temp1}
\begin{split}
&\mathbb{E}[||w^{L_m}_{t+1}-w^{*}||^2] + \mathbb{E}[||w^{G}_{t+1}-w^{*}||^2] \\
&\leq (1-2p_{t,L}^{L_m}\eta_{t}^{L_m}\gamma)\mathbb{E}[||w^{L_m}_{t}-w^{*}||^2] \\
&+(1-2p_{t,L}^{G}\eta_{t}^{L_m}\gamma)\mathbb{E}[||w^{G}_{t+\frac{1}{2}}-w^{*}||^2]\\
&+ [8(\eta_{t}^{L_m})^2+\frac{2\eta_{t}^{L_m}}{\gamma}]\times\\
&\frac{2b^2(2+12\sqrt{2})^2B^2}{n_{m}^2}(\frac{t-1}{b}+\frac{(t-1)^2}{n_{m}-(t-1)b})\\
&+8(\eta_{t}^{L_m})^2G^2 + 2p_{t,G}^{L_{m}}\mathbb{E}[||N_{t}^{m}||^2],
\end{split}
\end{equation}
\noindent in which $n_{m}=\frac{n}{M}$ is the total number of training samples that node $m$ has. Let $a=\min\{p_{t,L}^{L_1},p_{t,G}^{L_1},\cdots,p_{t,L}^{L_M},p_{t,G}^{L_M}\}$, we have
\vspace{2mm}
\begin{equation}
\begin{split}
&(1-2p_{t,L}^{L_m}\eta_{t}^{L_m}\gamma)\mathbb{E}[||w^{L_m}_{t}-w^{*}||^2]\\
&+(1-2p_{t,G}^{L_m}\eta_{t}^{L_m}\gamma)\mathbb{E}[||w^{G}_{t+\frac{1}{2}}-w^{*}||^2]\\
&\leq (1-2a\eta_{t}^{L_m}\gamma)\big[\mathbb{E}[||w^{L_m}_{t}-w^{*}||^2]+\mathbb{E}[||w^{G}_{t+\frac{1}{2}}-w^{*}||^2]\big].
\end{split}
\end{equation}

Let $v_{m}(t) = [8(\eta_{t}^{L_m})^2+\frac{2\eta_{t}^{L_m}}{\gamma}]\times\frac{2b^2(2+12\sqrt{2})^2B^2}{n_{m}^2}(\frac{t-1}{b}+\frac{(t-1)^2}{n_{m}-(t-1)b})+8(\eta_{t}^{L_m})^2G^2 + 2p_{t,G}^{L_{m}}\mathbb{E}[||N_{t}^{m}||^2]$ and assume that at time $t$, the global model has already been updated by another node $m'$, according to (\ref{Temp1}) we have
\vspace{2mm}
\begin{equation}\label{Temp2}
\begin{split}
&\mathbb{E}[||w^{L_{m'}}_{t+1}-w^{*}||^2] + \mathbb{E}[||w^{G}_{t+\frac{1}{2}}-w^{*}||^2] \leq v_{m'}(t) \\
&+(1-2a\eta_{t}^{L_{m'}}\gamma)\big[\mathbb{E}[||w^{L_{m'}}_{t}-w^{*}||^2]+\mathbb{E}[||w^{G}_{t}-w^{*}||^2]\big].
\end{split}
\end{equation}

Combing (\ref{Temp1}) and (\ref{Temp2}), we have
\begin{equation}\label{Temp3}
\begin{split}
&\mathbb{E}[||w^{L_m}_{t+1}-w^{*}||^2] + \mathbb{E}[||w^{L_{m'}}_{t+1}-w^{*}||^2] + \mathbb{E}[||w^{G}_{t+1}-w^{*}||^2]\\
&\leq (1-2a\eta_{t}^{L_m}\gamma)\big[\mathbb{E}[||w^{L_m}_{t}-w^{*}||^2]+\mathbb{E}[||w^{G}_{t+\frac{1}{2}}-w^{*}||^2]\big] \\
&+v_{m}(t) + \mathbb{E}[||w^{L_{m'}}_{t+1}-w^{*}||^2] \\
&\leq (1-2a\eta_{t}^{L_m}\gamma)\mathbb{E}[||w^{L_m}_{t}-w^{*}||^2] + v_{m}(t)+v_{m'}(t) \\
&+ (1-2a\eta_{t}^{L_{m'}}\gamma)\big[\mathbb{E}[||w^{L_{m'}}_{t}-w^{*}||^2]+\mathbb{E}[||w^{G}_{t}-w^{*}||^2]\big].
\end{split}
\end{equation}

By taking $\eta_{t}^{L_m} = \frac{1}{a\gamma (t-1)}, \forall m$, and extending (\ref{Temp3}) to the $M$ nodes case, we have
\begin{equation}\label{Temp4}
\begin{split}
&\sum_{m=1}^{M}\mathbb{E}[||w^{L_m}_{t+1}-w^{*}||^2] + \mathbb{E}[||w^{G}_{t+1}-w^{*}||^2] \leq \sum_{m=1}^{M}v_{m}(t) \\
&+(1-\frac{2}{t})\big[\sum_{m=1}^{M}\mathbb{E}[||w^{L_m}_{t}-w^{*}||^2]+\mathbb{E}[||w^{G}_{t}-w^{*}||^2]\big].
\end{split}
\end{equation}

In addition, when $\eta_{t}^{L_m} = \frac{1}{a\gamma (t-1)}, \forall m$,

\begin{equation}
\begin{split}
&[8(\eta_{t}^{L_m})^2+\frac{2\eta_{t}^{L_m}}{\gamma}]\times\\
&\frac{2b^2(2+12\sqrt{2})^2B^2}{n_{m}^2}(\frac{t-1}{b}+\frac{(t-1)^2}{n_{m}-(t-1)b})\\
&\leq 2(2+12\sqrt{2})^2B^2[\frac{8}{a^2\gamma^2 t}+\frac{2}{a\gamma^2}](\frac{1}{n_{m}[T-(t-1)]}),
\end{split}
\end{equation}

Since each node updates once at every time instance, $T = \frac{n}{bM}=\frac{n_m}{b}$. Therefore,
\begin{equation}
\begin{split}
&\sum_{m=1}^{M}\mathbb{E}[||w^{L_m}_{t+1}-w^{*}||^2] + \mathbb{E}[||w^{G}_{t+1}-w^{*}||^2] \\
&\leq (1-\frac{2}{t})\big[\sum_{m=1}^{M}\mathbb{E}[||w^{L_m}_{t}-w^{*}||^2]+\mathbb{E}[||w^{G}_{t}-w^{*}||^2]\big]\\
&+2M(2+12\sqrt{2})^2B^2[\frac{8}{a^2\gamma^2 t}+\frac{2}{a\gamma^2}](\frac{1}{n_{m}[T-(t-1)]})\\
&+\frac{8M}{a^2\gamma^2t^2}G^2+2\sum_{m=1}^{M}\mathbb{E}[||N_{t}^{m}||^2].
\end{split}
\end{equation}
According to Theorem \ref{TheoremA2Privacy},
\begin{equation}
\mathbb{E}[||N_{t}^{m}||^2] = \frac{8\ln(\frac{1.25}{\delta})L^2}{b^2a^2\epsilon^2\gamma^2(t-j)^2},
\end{equation}
\noindent in which $t-j-1$ is the time instance that node $m$ updates the global model.

In addition,
\begin{equation}
\begin{split}
&\frac{1}{n_{m}}(\frac{1}{t(T-t)}+\sum_{j=3}^{t}(\prod_{i=j}^{t}\frac{i-2}{i})\frac{1}{(j-1)(T-(j-2))})\\
&\leq \frac{1}{n_{m}t(t-1)}\int_{1}^{t}\frac{x}{T-x}dx \leq \frac{T\log t}{n_{m}t^2} = \frac{\log t}{bt^2}.
\end{split}
\end{equation}

\begin{equation}
\begin{split}
&\frac{1}{n_{m}}[\frac{1}{T-t}+\sum_{j=3}^{t}(\prod_{i=j}^{t}\frac{i-2}{i})\frac{1}{T-(j-2)}]\\
&\leq \frac{1}{n_{m}t}\int_{1}^{t}\frac{x}{T-x}dx \leq \frac{\log t}{bt}
\end{split}
\end{equation}
Therefore, by induction, we have
\begin{equation}
\begin{split}
&\sum_{m=1}^{M}\mathbb{E}[||w^{L_m}_{t+1}-w^{*}||^2] + \mathbb{E}[||w^{G}_{t+1}-w^{*}||^2] \\
&\leq \mathcal{O}(\frac{MG^2}{a^2t}+\frac{MB^2\log t}{a^2bt}+\frac{ML^2\ln(\frac{1.25}{\delta})}{a^2b^2\epsilon^2t}).
\end{split}
\end{equation}
\end{proof}

\section{Transductive Rademacher Complexity}\label{TRC}
\noindent We introduce some notion of transductive Rademacher complexity \cite{el2009transductive} that will be used in the convergence analysis of the proposed differentially private SGD algorithms.

\begin{Definition}
  Let $\mathcal{V}$ be a set of vectors $\boldsymbol{v}=(v_1,\cdots,v_n)$ in $\mathbb{R}^{n}$. Let $s,u$ be positive integers such that $s+u=n$, and denote $p = \frac{su}{(s+u)^2}\in(0,0.5)$. We define the transduction Rademacher Complexity $\mathcal{R}_{s,u}(\mathcal{V})$ as
  \begin{equation}
    \mathcal{R}_{s,u}(\mathcal{V}) = (\frac{1}{s}+\frac{1}{u})\mathbb{E}_{r_1,\cdots,r_n}(sup_{\boldsymbol{v}\in\mathcal{V}}\sum_{i=1}^{n}r_{i}v_{i}),
  \end{equation}
  where $r_1,\cdots,r_n$ are $i.i.d.$ random variables such that $r_i=1$ and $r_i=-1$ with probability $p$, $r_i=0$ with probability $1-2p$
\end{Definition}

\begin{Lemma}\label{lemma1}
  \cite{shamir2016without} Let $\mathcal{V}={v_i,i\in[n];v_i\leq B}$, we have $\mathcal{R}_{s,u}(\mathcal{V})\leq \sqrt{2}(\frac{1}{\sqrt{s}}+\frac{1}{\sqrt{u}})B$.
\end{Lemma}

\begin{Lemma}\label{lemmameng1}
  \cite{meng2017convergence} Let $\alpha$ be a random permutation over ${1,\cdots,n}$ chosen uniformly at random variables conditioned on $\alpha(1),\cdots,\alpha(tb)$, which are independent of $\alpha(tb+1),\cdots,\alpha(n)$. Let $s_{a:b}=\frac{1}{b+1-a}\sum_{i=a}^{b}s_i$. Then, we have $\forall t > 1$,
  \begin{equation}
    \mathbb{E}[\frac{1}{n}\sum_{i=1}^{n}s_i-\frac{1}{b}\sum_{j=1}^{b}s_{\sigma(tb+j)}]=\frac{tb}{n}\mathbb{E}[s_{1:tb}-s_{tb+1:n}].
  \end{equation}
\end{Lemma}

\begin{Lemma}\label{lemmameng2}
  \cite{meng2017convergence} Suppose $S \subset [-B,B]^{n}$ for some $B>0$. Let $\alpha$ be a random permutation over ${1,\cdots,n}$. Then we have
  \small
  \begin{equation}
    \mathbb{E}(sup_{s\in \mathcal{S}}(s_{1:tb}-s_{tb+1:n}))\leq \mathcal{R}_{tb,n-tb}(\mathcal{S}) + \\ \nonumber 12B(\frac{1}{\sqrt{tb}}+\frac{1}{\sqrt{n-tb}}),
  \end{equation}

  \begin{eqnarray}
    \sqrt{\mathbb{E}[sup_{s\in \mathcal{S}}(s_{1:tb}-s_{tb+1:n})]^2}\leq \sqrt{2}\mathcal{R}_{tb,n-tb}(\mathcal{S}) + \\ \nonumber 12\sqrt{2}B(\frac{1}{\sqrt{tb}}+\frac{1}{\sqrt{n-tb}}).
  \end{eqnarray}
  \normalsize
\end{Lemma}

\bibliography{Ref-Richeng}
\bibliographystyle{IEEEtran} 

\end{document}